\documentclass{esannV2}

\usepackage{graphicx,url,hyperref}
\usepackage{xcolor,wrapfig,algorithmic}
\usepackage[scriptsize,tight]{subfigure}
\usepackage{tabularx}
\usepackage{booktabs}
\usepackage{amssymb,amsmath,mathtools}
\usepackage{amsthm}
\usepackage{xspace}
\usepackage{listings}

\usepackage{amsthm}

\newtheorem{theorem}{Theorem}
\newtheorem{corollary}[theorem]{Corollary}
\newcolumntype{C}[1]{>{\centering\arraybackslash}p{#1}}
\newcolumntype{R}[1]{>{\raggedleft\arraybackslash}p{#1}}
\newcolumntype{L}[1]{>{\raggedright\arraybackslash}p{#1}}

\newcommand{\vk}{v_2}
\newcommand{\Cmap}{C}

\begin{document}

\date{}

\title{Adaptive Search in Collatz Exponent-Code Space via 2-adic and 3-adic Constraints}

\author{
Oliver Kramer
\vspace{0.3cm}\\
Computational Intelligence\\
Department of Computer Science\\
University of Oldenburg\\
\url{oliver.kramer@uol.de}
}

\maketitle
\thispagestyle{empty}

\begin{abstract}
We study a symbolic search space for the Collatz conjecture: finite exponent
codes of the accelerated map. Such a code records the number of divisions by
two after each \(3n+1\) step. Every code induces three diagnostics: real drift,
a 2-adic start representative, and a 3-adic endpoint representative. We call
their combination the 2--3--\(\infty\) diagnostic. A counterexample-like code
should have near-critical drift, small 2-adic start representatives, and
endpoint representatives compatible with the growth bound \((3/2)^k\). We
prove that every infinite code generated by a fixed positive integer has
asymptotically vanishing 2-adic and 3-adic residue rates. Experiments compare
random critical codes, mechanical critical codes, and adaptive evolutionary
search for \(k=100,200,400\). Adaptive search improves finite-length trade-offs,
but all methods retain clearly positive residue rates. The contribution is not
a Collatz verification method, but a symbolic diagnostic framework for probing
obstruction structures in exponent-code space.
\end{abstract}

\section{Introduction}

The Collatz conjecture states that repeated application of \(T(n)=n/2\) for
even \(n\), and \(T(n)=3n+1\) for odd \(n\), eventually reaches \(1\) for every
positive integer \(n\). Most computational approaches test large ranges of
starting values. Here, we instead search over symbolic trajectory structures.

We use the accelerated Collatz map on odd integers,
\[
\Cmap(n)=\frac{3n+1}{2^{\vk(3n+1)}} ,
\]
where \(\vk(m)\) is the largest exponent such that \(2^{\vk(m)}\mid m\).
For a trajectory \(x_{k+1}=\Cmap(x_k)\), define
\(a_{k+1}=\vk(3x_k+1)\). The finite sequence
\((a_1,\ldots,a_k)\) is an exponent code.

\paragraph{Example.}
For \(x_0=5\), one obtains \(3\cdot5+1=16=2^4\), hence \(a_1=4\) and
\(x_1=1\). Since \(3\cdot1+1=4=2^2\), the accelerated orbit then stays at
\(1\), giving the code
\[
(4,2,2,2,\ldots).
\]
Thus exponent codes encode odd-to-odd Collatz dynamics symbolically.

A finite code determines a start residue modulo \(2^{A_k}\), where
\(A_k=a_1+\cdots+a_k\), and an endpoint residue modulo \(3^k\). A
counterexample-like code should therefore satisfy three compatibility
conditions: near-critical average exponent, small forced 2-adic start
representative, and 3-adic endpoint representative compatible with real growth.
We call this the 2--3--\(\infty\) diagnostic.

\section{Related Work}

The (3x+1) problem has been studied from probabilistic, dynamical-systems,
and symbolic perspectives. Standard references include the survey by
Lagarias~\cite{lagarias1985}, the reference volume edited by
Lagarias~\cite{lagarias2010}, and the dynamical-systems treatment of
Wirsching~\cite{wirsching1998}. The accelerated map on odd integers and
symbolic trajectory encodings are well established in this literature.

Our work differs by treating finite exponent codes themselves as search
objects. Rather than exploring starting integers, we explore symbolic
trajectory structures and evaluate them through a combination of real-valued
drift, 2-adic start representatives, and 3-adic endpoint representatives.

The adaptive component is inspired by evolutionary computation, particularly
evolution strategies~\cite{beyer2002} and evolutionary optimization methods
\cite{eiben2003}. The goal is not to prove a number-theoretic statement by
optimization, but to use adaptive search as a tool for exploring structural
properties of exponent-code space.

\section{Exponent-Code Diagnostics}

We consider accelerated trajectories
\[
x_{k+1}=\Cmap(x_k)=\frac{3x_k+1}{2^{a_{k+1}}},
\qquad a_{k+1}=\vk(3x_k+1).
\]
For a finite code \((a_1,\ldots,a_k)\), define
\(A_k=\sum_{i=1}^{k}a_i\) and \(A_0=0\). Repeated substitution gives
\[
x_k=\frac{3^k x_0+B_k}{2^{A_k}},
\qquad
B_k=\sum_{j=0}^{k-1}3^{k-1-j}2^{A_j}.
\]
Thus the code controls both the factor \(3^k/2^{A_k}\) and the additive offset
\(B_k\).

\subsection{Real Drift}

The critical average exponent is \(\alpha=\log_2 3\), since \(A_k/k=\alpha\)
makes \(3^k/2^{A_k}=1\). Codes with larger averages are contracting; codes with
smaller averages are expanding. We use
\[
d_k=\left|A_k/k-\alpha\right|
\]
as drift diagnostic.

\subsection{The 2-adic Start Representative}

A code also constrains possible starting values. From the affine expression,
any realizing start \(x_0\) must satisfy
\[
3^kx_0+B_k\equiv 0 \pmod{2^{A_k}} .
\]
Since \(3^k\) is invertible modulo \(2^{A_k}\), define the least nonnegative
representative \(r_k\) by
\[
r_k\equiv
-\sum_{j=0}^{k-1}3^{-(j+1)}2^{A_j}
\pmod{2^{A_k}},
\qquad 0\leq r_k<2^{A_k}.
\]
If an infinite code is generated by a fixed positive integer \(n\), then
eventually \(r_k=n\). We measure start-residue growth by
\[
\rho_r(k)=\frac{\log(1+r_k)}{k}.
\]

\subsection{The 3-adic Endpoint Representative}

The code also determines an endpoint representative. Since
\(2^{A_k}x_k\equiv B_k \pmod{3^k}\), define \(M_k\) by
\[
M_k\equiv
2^{-A_k}
\sum_{j=0}^{k-1}3^{k-1-j}2^{A_j}
\pmod{3^k},
\qquad 1\leq M_k\leq 3^k.
\]
If the code is realized by a positive orbit starting at \(n\), then
\(x_k\equiv M_k \pmod{3^k}\). Moreover,
\(x_{k+1}+1\leq \frac32(x_k+1)\), hence
\(x_k+1\leq(n+1)(3/2)^k\). Thus a genuine orbit eventually satisfies
\(M_k=O((3/2)^k)\). We measure endpoint incompatibility by
\[
\rho_M(k)=
\frac{\log\left(1+M_k/(3/2)^k\right)}{k}.
\]
The quantities \(d_k\), \(\rho_r(k)\), and \(\rho_M(k)\) form the finite-prefix
2--3--\(\infty\) diagnostic.

\subsection{Necessary Vanishing of Residue Rates}

The residue diagnostics are necessary compatibility tests: a genuine orbit
must eventually have a fixed start representative and an endpoint bounded by
the real growth rate.

\begin{theorem}[Necessary vanishing of the residue rates]
\label{thm:vanishing-rates}
Let \((a_1,a_2,\ldots)\) be the infinite exponent code generated by the
accelerated Collatz orbit of a fixed positive odd integer \(n\). Let \(r_k\)
and \(M_k\) be the representatives induced by the prefix
\((a_1,\ldots,a_k)\). Then
\[
\lim_{k\to\infty}\rho_r(k)=0
\qquad\text{and}\qquad
\lim_{k\to\infty}\rho_M(k)=0.
\]
\end{theorem}

\begin{proof}
Because the code is generated by \(x_0=n\), the start congruence is satisfied
by \(n\), so \(r_k\equiv n \pmod{2^{A_k}}\). Since \(A_k\geq k\), eventually
\(n<2^{A_k}\). Hence the least nonnegative representative is \(r_k=n\), and
\[
\rho_r(k)=\frac{\log(1+n)}{k}\to0.
\]

For the endpoint, the realized orbit satisfies \(x_k\equiv M_k \pmod{3^k}\).
Since \(a_{k+1}\geq1\),
\[
x_{k+1}+1\leq \frac32(x_k+1),
\]
and therefore \(x_k+1\leq(n+1)(3/2)^k\). Since
\[
\frac{(n+1)(3/2)^k}{3^k}=\frac{n+1}{2^k}\to0,
\]
we have \(x_k<3^k\) for all sufficiently large \(k\), so \(M_k=x_k\)
eventually. Consequently,
\[
\frac{M_k}{(3/2)^k}\leq n+1
\]
for all sufficiently large \(k\), and
\[
\rho_M(k)\leq \frac{\log(n+2)}{k}\to0.
\]
\end{proof}
\begin{corollary}
\label{cor:positive-liminf}
If an infinite exponent code satisfies
\(\liminf_{k\to\infty}\rho_r(k)>0\) or
\(\liminf_{k\to\infty}\rho_M(k)>0\), then it cannot be generated by the
accelerated Collatz orbit of any fixed positive integer.
\end{corollary}

The 2--3--\(\infty\) diagnostic combines three viewpoints. The drift \(d_k\)
is the real component and measures whether \(3^k/2^{A_k}\) is near the neutral
threshold. The representative \(r_k\) is the 2-adic component and measures how
large the forced start residue is. The representative \(M_k\) is the 3-adic
component and measures whether the forced endpoint residue is compatible with
the real growth bound \((3/2)^k\).
For a genuine positive orbit, Theorem~\ref{thm:vanishing-rates} implies
\[
\rho_r(k)\to0,\qquad \rho_M(k)\to0.
\]
Thus positive residue rates indicate incompatibility with any fixed finite
starting value. The central question is therefore not merely whether one can
construct near-critical exponent averages, but whether near-criticality can be
combined with simultaneous 2-adic and 3-adic compatibility.

\section{Adaptive Search and Baselines}

We compare three strategies. The first baseline samples random critical binary
codes with \(a_i\in\{1,2\}\) and
\(\Pr(a_i=2)=\log_2 3-1\), so that
\(\mathbb{E}[a_i]=\log_2 3\).
The second baseline samples mechanical critical codes. Let
\(\beta=\log_2 3-1\). For random \(\theta\in[0,1)\), generate $
a_i=
1+\lfloor i\beta+\theta\rfloor
-\lfloor(i-1)\beta+\theta\rfloor .$
These codes distribute the \(2\)-exponents evenly and are nearly critical by
construction.
The third strategy is adaptive evolutionary search over fixed-length codes
with \(a_i\in\{1,\ldots,a_{\max}\}\). Each code is a genome. Selection favors
low diagnostic score, crossover recombines segments, mutations introduce
variation, and drift repair moves candidates toward \(A_k/k\approx\log_2 3\).
Mechanical-block injection adds structured critical fragments.
Candidates are ranked by $S=d_k+\rho_r(k)+\rho_M(k)$,
where lower values indicate more counterexample-like codes. The score is used
only as a finite-prefix search diagnostic, not as a proof criterion.

\section{Experiments}

We evaluate all methods at \(k\in\{100,200,400\}\) with budget
\(N_{\mathrm{eval}}=1000k\). For every candidate, we compute \(A_k/k\), \(d_k\),
\(\rho_r(k)\), \(\rho_M(k)\), and \(S\). Table~\ref{tab:scaling} reports
representative best-of-run values under equal evaluation budgets.

\begin{table}[htb!]
\centering
\setlength{\tabcolsep}{6pt}
\renewcommand{\arraystretch}{1.08}
\begin{tabular}{llrrrr}
\toprule
\(k\) & Method & \(d_k\) & \(\rho_r\) & \(\rho_M\) & \(S\) \\
\midrule
100 & Random critical & 0.03496 & 0.96950 & 0.58827 & 1.59274 \\
100 & Mechanical      & 0.00496 & 1.03977 & 0.63774 & 1.68247 \\
100 & Adaptive search & 0.00504 & 0.94940 & 0.54044 & 1.49488 \\
\midrule
200 & Random critical & 0.01996 & 1.02255 & 0.63092 & 1.67344 \\
200 & Mechanical      & 0.00004 & 1.06915 & 0.66366 & 1.73284 \\
200 & Adaptive search & 0.00004 & 1.04606 & 0.64057 & 1.68666 \\
\midrule
400 & Random critical & 0.00246 & 1.06891 & 0.66515 & 1.73653 \\
400 & Mechanical      & 0.00004 & 1.08330 & 0.67781 & 1.76115 \\
400 & Adaptive search & 0.00754 & 1.07076 & 0.66007 & 1.73838 \\
\bottomrule
\end{tabular}
\caption{Scaling behavior for fixed code lengths. Each method uses \(1000k\)
candidate evaluations. Lower values are better.}
\label{tab:scaling}
\end{table}

At \(k=100\), adaptive search obtains the best score, \(S=1.49488\), and the
lowest residue rates. At \(k=200\), random critical sampling obtains the best
unweighted score, while adaptive search preserves nearly exact drift and
improves over the mechanical baseline. At \(k=400\), random critical sampling
and adaptive search are very close.
The strongest trend is persistence of the diagnostic obstruction. Across all
lengths and methods, the rates remain clearly positive, with
\(0.95\leq\rho_r\leq1.08\) and \(0.54\leq\rho_M\leq0.68\). None of the methods
approaches the necessary regime \(\rho_r(k)\to0\) and \(\rho_M(k)\to0\).
Adaptive search improves its initial populations, but the improvements become
smaller as \(k\) grows.

\begin{figure}[t]
\centering
\includegraphics[width=0.83\linewidth]{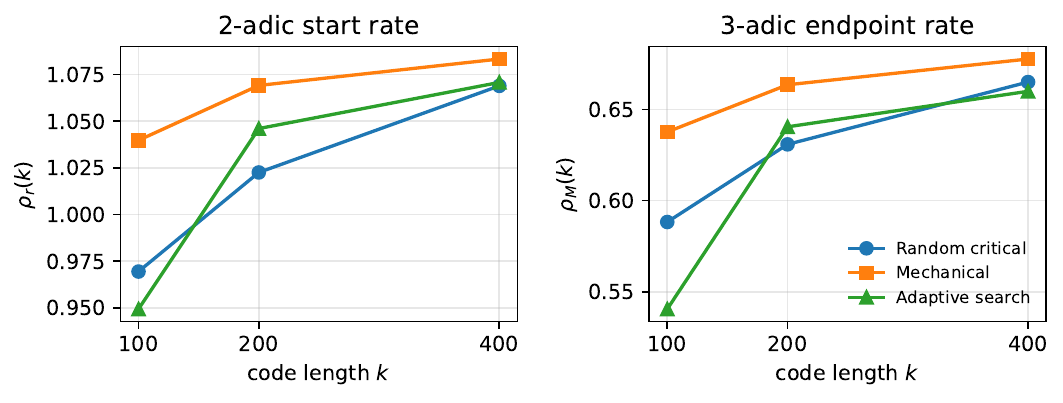}
\caption{Scaling of the diagnostic residue rates. Left: 2-adic start rate
\(\rho_r\). Right: 3-adic endpoint rate \(\rho_M\). Both rates remain clearly
positive over the investigated length range.}
\label{fig:scaling-rates}
\end{figure}

Figure~\ref{fig:scaling-rates} shows that the residue rates do not trend toward
zero over the investigated range. Thus, although adaptive search uses
recombination, mutation, drift repair, and domain-informed variation, it does
not escape the positive 2-adic and 3-adic rates observed in simpler baselines.

\section{Conclusion}

We introduced a symbolic framework for accelerated Collatz exponent codes based
on real drift and 2-adic and 3-adic representatives. We proved that codes
generated by fixed positive integers satisfy
\[
\rho_r(k),\rho_M(k)\to 0.
\]
Experiments for \(k=100,200,400\) show persistently positive residue rates.
Adaptive search improves finite-length results but does not approach this
necessary regime. This does not prove the Collatz conjecture, but indicates that
simultaneous real, 2-adic, and 3-adic compatibility is difficult to achieve.
Future work should examine larger lengths and provable lower bounds.

\begin{footnotesize}

\end{footnotesize}

\end{document}